\documentclass[twoside,11pt]{article}

\usepackage{jmlr2e}

\usepackage{makecell}

\usepackage[utf8]{inputenc}

\usepackage{cite}

\usepackage{nameref}

\usepackage[right]{lineno}
\usepackage{algorithm}
\usepackage{algpseudocode}

\usepackage{microtype}
\DisableLigatures[f]{encoding = *, family = * }

\usepackage{changepage}

\makeatletter
\renewcommand{\@biblabel}[1]{\quad#1.}
\makeatother

\usepackage{fancyhdr,graphicx}
\usepackage{epstopdf}

\usepackage{color}

\definecolor{Gray}{gray}{.25}

\usepackage{graphicx}

\usepackage{sidecap}

\usepackage{wrapfig}
\usepackage[pscoord]{eso-pic}
\usepackage[fulladjust]{marginnote}
\reversemarginpar

\usepackage{graphicx} 
\usepackage{subcaption}
\usepackage{enumerate}

\usepackage{tikz}
\usepackage{xcolor}
\usetikzlibrary{arrows}

\usepackage{mathrsfs}
\usepackage{bbm}

\usepackage{bm,todonotes}

\usepackage{thmtools}
\usepackage{thm-restate}

\usepackage[utf8]{inputenc} 
\usepackage[T1]{fontenc}    
\usepackage{booktabs}       
\usepackage{amsfonts}       
\usepackage{nicefrac}       
\usepackage{microtype}      
\usepackage{amssymb}
\usepackage{graphicx}
\usepackage{color}
\usepackage{amsmath}
\usepackage{pifont}
\usepackage{tabularx}
\usepackage{fancyvrb}
\usepackage{comment}
\usepackage{float}
\usepackage{wrapfig}
\usepackage{mathtools}
\usepackage{thmtools,thm-restate}
\usepackage{paralist} 
\usepackage{multirow}
\usepackage{scrextend}
\usepackage{enumitem}
\usepackage[noorphans]{quoting}
\usepackage{bm}
\usepackage{bbm}
\usepackage{tabularx}
\usepackage{xcolor}
\usepackage{xspace}
\usepackage{mathpazo}
\usepackage{times}
\usepackage{tabu}
\usepackage{icomma}
\usepackage{algpseudocode}
\usepackage{listings}
\usepackage{amssymb}
\usepackage{pifont}

\newcommand{\fairopt}{h_{\text{Fair}}^*}
\newcommand{\cmark}{\ding{51}}%
\newcommand{\xmark}{\ding{55}}%

\DeclareFontFamily{U}{mathx}{\hyphenchar\font45}
\DeclareFontShape{U}{mathx}{m}{n}{
	<5> <6> <7> <8> <9> <10> gen * mathx
	<10.95> mathx10 <12> <14.4> <17.28> <20.74> <24.88> mathx12
}{}
\DeclareSymbolFont{mathx}{U}{mathx}{m}{n}
\DeclareMathSymbol{\intop}  {1}{mathx}{"B3}

\let\temp\phi
\let\phi\varphi
\let\varphi\temp

\newcommand{\ie}{i.e.}

\newcommand{\dist}{\mu}
\newcommand{\distp}{\mathcal{P}}
\newcommand{\distq}{\mathcal{Q}}
\newcommand{\distm}{\mathcal{M}}
\newcommand{\Hzo}{H_{0\mbox{-}1}}

\newcommand{\eps}{\varepsilon}

\newcommand{\opt}{\textbf{OPT}}

\DeclarePairedDelimiterX{\infdivx}[2]{(}{)}{%
	#1\;\delimsize\|\;#2%
}

\DeclareMathOperator{\tr}{tr}

\DeclareMathOperator*{\argmin}{arg\,min}

\newcommand{\defeq}{\vcentcolon=}

\newcommand{\xxspace}{\mathcal{X}}
\newcommand{\yyspace}{\mathcal{Y}}
\newcommand{\zzspace}{\mathcal{Z}}
\newcommand{\aaspace}{\mathcal{A}}
\newcommand{\Ypred}{\widehat{Y}}

\newcommand{\err}{\mathrm{Err}}

\newcommand{\RR}{\mathbb{R}}

\newcommand{\Nat}{\mathbb{N}}
\newcommand{\Exp}{\mathbb{E}}
\newcommand{\HH}{\mathcal{H}}
\newcommand{\xx}{\mathbf{x}}

\DeclarePairedDelimiterX{\inp}[2]{\langle}{\rangle}{#1, #2}
\newcommand{\djs}{d_{\text{JS}}}
\newcommand{\jsd}{D_{\text{JS}}}
\newcommand{\dtv}{d_{\text{TV}}}

\newcommand{\kl}{D_{\text{KL}}}
\newcommand{\dbr}{\Delta_{\mathrm{BR}}}
\newcommand{\dpgap}{\Delta_{\mathrm{DP}}}

\definecolor{maroon}{rgb}{0.5, 0.0, 0.0}

\makeatletter
\def\thm@space@setup{\thm@preskip=2pt
	\thm@postskip=0pt}
\makeatother

\usepackage{lastpage}
\jmlrheading{23}{2022}{1-\pageref{LastPage}}{11/21}{1/22}{21-1427}{Han Zhao, Geoffrey J. Gordon}
\ShortHeadings{title}{Zhao, Gordon}

\ShortHeadings{Inherent Tradeoffs in Learning Fair Representations}{Zhao and Gordon}
\firstpageno{1}

\begin{document}

\title{Inherent Tradeoffs in Learning Fair Representations}

\author{
       \AND
       \name Han Zhao\thanks{This work is an extended version of an earlier paper with the same title appearing in NeurIPS 2019 by the same authors.} \email hanzhao@illinois.edu \\
       \addr{University of Illinois at Urbana-Champaign}
       \AND 
       Geoffrey J. Gordon \email ggordon@cs.cmu.edu\\
       \addr{Carnegie Mellon University}
}

\editor{Maya Gupta}

\maketitle

\begin{abstract}
    Real-world applications of machine learning tools in high-stakes domains are often regulated to be fair, in the sense that the predicted target should satisfy some quantitative notion of parity with respect to a protected attribute. However, the exact tradeoff between fairness and accuracy is not entirely clear, even for the basic paradigm of classification problems. In this paper, we characterize an inherent tradeoff between statistical parity and accuracy in the classification setting by providing a lower bound on the sum of group-wise errors of any fair classifiers. Our impossibility theorem could be interpreted as a certain uncertainty principle in fairness: if the base rates differ among groups, then any fair classifier satisfying statistical parity has to incur a large error on at least one of the groups. We further extend this result to give a lower bound on the joint error of any (approximately) fair classifiers, from the perspective of learning fair representations. To show that our lower bound is tight, assuming oracle access to Bayes (potentially unfair) classifiers, we also construct an algorithm that returns a randomized classifier which is both optimal (in terms of accuracy) and fair. Interestingly, when the protected attribute can take more than two values, an extension of this lower bound does not admit an analytic solution. Nevertheless, in this case, we show that the lower bound can be efficiently computed by solving a linear program, which we term as the TV-Barycenter problem, a barycenter problem under the TV-distance.
    
    On the upside, we prove that if the group-wise Bayes optimal classifiers are close, then learning fair representations leads to an alternative notion of fairness, known as the accuracy parity, which states that the error rates are close between groups. Finally, we also conduct experiments on real-world datasets to confirm our theoretical findings.
\end{abstract}


\section{Introduction}
\label{sec:intro}
With the prevalence of machine learning applications in high-stakes domains, e.g., criminal judgement, medical testing, online advertising, etc., it is crucial to ensure that the automated decision making systems do not propagate existing bias or discrimination that might exist in historical data~\citep{barocas2016big,berk2018fairness}. Among many recent proposals for achieving different notions of algorithmic fairness~\citep{zemel2013learning,dwork2012fairness,zafar2015fairness,hardt2016equality,zafar2017fairness}, learning fair representations has received increasing attention due to recent advances in learning rich representations with deep neural networks~\citep{edwards2015censoring,louizos2015variational,madras2018learning,zhang2018mitigating,beutel2017data,song2019fair,zhao2020conditional,chi2021understanding}. In fact, a line of work has proposed to learn group-invariant representations with adversarial learning techniques in order to achieve statistical parity, also known as the demographic parity in the literature. This line of work dates at least back to~\citet{zemel2013learning} where the authors proposed to learn predictive models that are independent of the group membership attribute. At a high level, the underlying idea is that if representations of instances from different groups are similar to each other, then any predictive model on top of them will certainly make decisions independent of group membership. 

On the other hand, it has long been observed that there is an underlying tradeoff between accuracy and statistical parity. In particular, it is easy to see that in an extreme case where the group membership coincides with the target variable to predict, a call for exact statistical parity will inevitably remove the perfect predictor~\citep{hardt2016equality}. Empirically, it has also been observed that a tradeoff exists between accuracy and fairness in binary classification~\citep{zliobaite2015relation}. Clearly, methods based on learning fair representations are also bound by such inherent tradeoff between accuracy and fairness. But before attempting to develop an algorithm to achieve a particular goal on fairness, it is natural to ask:
\begin{enumerate}
    \item[\textbf{Q1}]   How does the fairness constraint trade for accuracy? Without further assumptions on the data generating distributions, what is the exact price any fair classifiers have to pay for fairness?
    \item[\textbf{Q2}]   Furthermore, given the underlying distribution, can we construct an algorithm to return the optimal (in terms of accuracy) fair classifier?  
    \item[\textbf{Q3}]   Will learning fair representations help to achieve other notions of fairness besides the statistical parity? If yes, what is the fundamental limit of accuracy that we can hope to achieve under such constraint?
\end{enumerate}

To answer the above questions, through the lens of information theory, in this paper we provide the first result that quantitatively characterizes the tradeoff between demographic parity and the sum of group-wise accuracy across different population groups. Specifically, when the base rates differ between groups, we provide a tight information-theoretic lower bound on the joint error across these groups. Our lower bound is algorithm-independent so it holds for all methods that satisfy statistical parity. We also extend this result to prove a lower bound on the joint accuracy for any fair classifiers, and generalize it to the case where the protected attribute can take any finite number of values. Interestingly, when the number of groups defined by the protected attribute is more than two, we can no long obtain an analytic lower bound. Nevertheless, we show that the lower bound can be efficiently computed by solving a linear program, which we term as the TV-Barycenter problem, a barycenter problem under the TV-distance. To show that our lower bound is tight, assuming oracle access to Bayes (potentially unfair) classifiers, we derive an algorithm that returns a randomized classifier which is both optimal (in terms of accuracy) and fair. 

When only approximate statistical parity is achieved, we present a family of lower bounds to quantify the tradeoff of accuracy introduced by such approximate constraint. As a side contribution, our proof technique is simple but general, and we expect it to have broader applications in other learning problems using adversarial techniques, e.g., unsupervised domain adaptation~\citep{ganin2016domain,zhao2019learning}, privacy-preservation under attribute inference attacks~\citep{hamm2017minimax,zhao2019adversarial} and multilingual machine translation~\citep{johnson2017google,zhao2020learning}.

To complement our negative results, we show that if the (potentially unfair) Bayes optimal classifiers across different groups are close, then learning fair representations helps to achieve an alternative notion of fairness, i.e., the accuracy parity~\citep{buolamwini2018gender}, which states that the error rates are close between different groups. Empirically, we conduct experiments on a real-world dataset to corroborate both our positive and negative results. We believe our theoretical insights contribute to better understanding of the tradeoff between accuracy and different notions of fairness, and they are also helpful in guiding the future design of representation learning algorithms to achieve algorithmic fairness.

\section{Preliminaries}
\label{sec:preliminary}
We first introduce the notation used throughout the paper and formally describe the problem setup. We then briefly discuss some information-theoretic concepts that will be used in our analysis.

\subsection{Notation}
We consider a general classification setting where there is a joint distribution $\dist$ over the triplet $T = (X, A, Y)$, where $X\in\xxspace\subseteq\RR^d$ is the input vector, $A\in\{0, 1\}$\footnote{Our main results could be extended to the case where $A$ can take finitely many values. We show this extension in Section~\ref{sec:multi}} is the protected attribute, e.g., race, gender, etc., and $Y\in\yyspace = \{0, 1\}$ is the target output. Lower case letters $\xx$, $a$ and $y$ are used to denote the instantiation of $X$, $A$ and $Y$, respectively. Let $\HH$ be a hypothesis class of predictors from input to output space. Throughout the paper, we focus on the setting where the classifier \emph{cannot} directly use the sensitive attribute $A$ to form its prediction. However, note that even if the classifier does not explicitly take the protected attribute $A$ as input, this \emph{fairness through blindness} mechanism can still be biased due to the redundant encoding issue~\citep{barocas2017fairness}. To keep the notation uncluttered, for $a\in\{0, 1\}$, we use $\dist_a$ to mean the conditional distribution of $\dist$ given $A = a$. We use $\dist(Y)$ to denote the marginal distribution of $Y$ from a joint distribution $\dist$ over $Y$ and some other random variables. With slight abuse of notation, occasionally we also use $Y_\sharp\dist$ to denote the marginal distribution of $Y$ from the joint distribution $\dist$, \ie, projection of $\dist$ onto the $Y$ coordinate. 

For an event $E$, $\dist(E)$ denotes the probability of $E$ under $\dist$. In particular, in the literature of fair machine learning, we call $\dist(Y = 1)$ the \emph{base rate} of distribution $\dist$ and we use $\dbr(\dist, \dist')\defeq |\dist(Y = 1) - \dist'(Y = 1)|$ to denote the difference of the base rates between two distributions $\dist$ and $\dist'$ over the same sample space. 

Given a feature transformation function $g:\xxspace\to\zzspace$ that maps instances from the input space $\xxspace$ to feature space $\zzspace$, we define $g_\sharp\dist\defeq \dist\circ g^{-1}$ to be the induced (pushforward) distribution of $\dist$ under $g$, i.e., for any event $E'\subseteq\zzspace$, $g_\sharp\dist(E') \defeq \dist(g^{-1}(E')) = \dist(\{x\in\xxspace\mid g(x)\in E'\})$. The zero-one entropy of $A$~\citep[Section 3.5.3]{grunwald2004game} is denoted as $\Hzo(A)\defeq 1 - \max_{a\in\{0, 1\}}\Pr(A = a)$. Furthermore, we use $F_{\dist}$ to represent the cumulative distribution function of $\dist$, i.e., for $z\in\RR$, $F_{\dist}(z)\defeq\Pr_\dist((-\infty, z])$. 

\subsection{Group Fairness}
\label{sec:gfair}
Given a joint distribution $\dist$, the error of a predictor $h$ under $\dist$ is defined as $\err_\dist(h)\defeq\Pr_\dist(Y \neq h(X))$. To make the notation more compact, we may drop the subscript $\dist$ when it is clear from the context. In this work we focus on group fairness where the group membership is given by the sensitive attribute $A$. Even in this context there are many possible definitions of \emph{fairness}~\citep{narayanan2018translation}, and in what follows we provide a brief review of the ones that are mostly relevant to this work.

\begin{definition}[Demographic Parity]
\label{def:demographic}
    Given a joint distribution $\dist$, a classifier $\Ypred$ satisfies \emph{demographic parity} if $\Ypred$ is independent of $A$.
\end{definition}
Demographic parity reduces to the requirement that $\dist_0(\Ypred = 1) = \dist_1(\Ypred = 1)$, i.e., positive outcome is given to the two groups at the same rate. When exact equality does not hold, we use the absolute difference between them as an approximate measure:
\begin{definition}[DP Gap]
    Given a joint distribution $\dist$, the \emph{demographic parity gap} of a classifier $\Ypred$ is $\dpgap(\Ypred)\defeq |\dist_0(\Ypred = 1) - \dist_1(\Ypred = 1)|$.
\end{definition}
Demographic parity is also known as \emph{statistical parity}, and it has been adopted as definition of fairness in a series of seminal works~\citep{calders2009building,edwards2015censoring,johndrow2019algorithm,kamiran2009classifying,kamishima2011fairness,louizos2015variational,zemel2013learning,madras2018learning}. However, as we shall quantify precisely in Section~\ref{sec:main}, demographic parity may reduce the accuracy that we hope to achieve, especially in the common scenario where the \emph{base rates} differ between two groups, e.g., $\dist_0(Y = 1)\neq \dist_1(Y = 1)$. In light of this, an alternative definition is \emph{accuracy parity}:
\begin{definition}[Accuracy Parity]
    Given a joint distribution $\dist$, a classifier $h$ satisfies \emph{accuracy parity} if $\err_{\dist_0}(h) = \err_{\dist_1}(h)$.
\end{definition}
In the literature, a break of accuracy parity is also known as disparate mistreatment~\citep{zafar2017fairness}. Again, when $h$ is a binary classifier, accuracy parity reduces to $\dist_0(h(X) = Y) = \dist_1(h(X) = Y)$. Different from demographic parity, the definition of accuracy parity does not eliminate the perfect predictor when $Y = A$ when the base rates differ between two groups. When costs of different error types matter, more refined definitions exist:
\begin{definition}[Equalized Odds~\citep{hardt2016equality}]
    Given a joint distribution $\dist$, a classifier $h$ satisfies \emph{equalized odds} if $\dist_0(h(X) = 1\mid Y = y) = \dist_1(h(X) = 1\mid Y = y)$, $\forall y\in\{0, 1\}$.
\end{definition}
Equalized odds essentially requires equal true positive and false positive rates between different groups. Furthermore, \citet{hardt2016equality} also defined \emph{true positive parity}, or \emph{equal opportunity}, to be $\dist_0(h(X) = 1\mid Y = 1) = \dist_1(h(X)=1\mid Y = 1)$ when the positive outcome is more desirable in certain applications. For example, in school admission, the cost of denying a competent candidate is considerably higher than the other way around. Last but not least, \emph{predictive rate parity}, also known as \emph{test fairness}~\citep{chouldechova2017fair}, asks for equal chance of positive outcomes across groups given predictions:
\begin{definition}[Predictive Rate Parity]
    Given a joint distribution $\dist$, a probabilistic classifier $h$ satisfies \emph{predictive rate parity} if $\dist_0(Y = 1\mid h(X) = c) = \dist_1(Y = 1\mid h(X) = c)$, $\forall c\in[0, 1]$.
\end{definition}
A closely related notion of predictive rate parity is known as \emph{statistical calibration}. Formally, a classifier $h$ is said to be \emph{calibrated} if $\dist(Y = 1\mid h(X) = c) = c, \forall c\in[0, 1]$, i.e., if we look at the set of data that receive a predicted probability of $c$ by $h$, we would like $c$-fraction of them to be positive instances according to $Y$~\citep{pleiss2017fairness}. Hence, it is clear to see that if a classifier $h$ is calibrated across different subgroups, then it also satisfies predictive rate parity.

In the special case when $h$ is a deterministic binary classifier that only takes value in $\{0, 1\}$, \citet{chouldechova2017fair} showed an intrinsic incompatibility between predictive rate parity and equalized odds:
\begin{theorem}[\citet{chouldechova2017fair}]
    Assume $\dist_0(Y = 1)\neq \dist_1(Y = 1)$, then for any deterministic classifier $h:\xxspace\to\{0, 1\}$ that is not perfect, i.e., $h(X)\neq Y$, equalized odds and predictive rate parity cannot hold simultaneously.
\end{theorem}

Similar incompatibility result for probabilistic classifier has also been proved by~\citet{kleinberg2016inherent}, where the authors showed that for any non-perfect predictors, statistical calibration and equalized odds cannot be achieved simultaneously if the base rates are different across groups. 

\subsection{$f$-divergence}
Introduced by~\citet{ali1966general} and~\citet{csiszar1964informationstheoretische,csiszar1967information}, $f$-divergence, also known as the Ali-Silvey distance, is a general class of statistical divergences to measure the difference between two probability distributions $\distp$ and $\distq$ over the same probability space. 
\begin{definition}[$f$-divergence]
    Let $\distp$ and $\distq$ be two probability distributions over the same space and assume $\distp$ is absolutely continuous w.r.t.\ $\distq$ ($\distp\ll\distq$). Then for any convex function $f:(0,\infty)\to\RR$ that is strictly convex at 1 and $f(1) = 0$, the $f$-divergence of $\distq$ from $\distp$ is defined as
    \begin{equation}
        \label{equ:fdiv}
        D_f(\distp~\|~\distq)\defeq \Exp_{\distq}\bigg[f\bigg(\frac{d\distp}{d\distq}\bigg)\bigg].
    \end{equation}
    The function $f$ is called the \emph{generator function} of $D_f(\cdot~\|~\cdot)$.
\end{definition}
Different choices of the generator function $f$ recover popular statistical divergence as special cases, e.g., the KL-divergence. From Jensen's inequality it is easy to verify that $D_f(\distp~\|~\distq)\geq 0$ and $D_f(\distp~\|~\distq) = 0$ iff $\distp = \distq$ almost surely. Note that $f$-divergence does not necessarily leads to a distance metric, and it is not symmetric in general, i.e., $D_f(\distp~\|~\distq)\neq D_f(\distq~\|~\distp)$ provided that $\distp\ll\distq$ and $\distq\ll\distp$. We list some common choices of the generator function $f$ and their corresponding properties in Table~\ref{tab:fdiv}. Notably,~\citet{khosravifard2007confliction} proved that among all the $f$-divergences, total variation is the only $f$-divergence that serves as a metric, i.e., satisfying the triangle inequality. 
\begin{table}[tb]
    \centering
    \setlength\tabcolsep{2pt}
    \small
    \caption{List of different $f$-divergences and their corresponding properties. $\kl(\distp~\|~\distq)$ denotes the KL-divergence of $\distq$ from $\distp$ and $\distm\defeq (\distp + \distq) / 2$ is the average distribution of $\distp$ and $\distq$. Symm.\ stands for symmetric and Tri.\ stands for triangle inequality.}
    \label{tab:fdiv}
    \begin{tabular}{*5l}\\\toprule
    \textbf{Name} & $D_f(\distp~\|~\distq)$ & \textbf{Generator} $f(t)$ & \textbf{Symm.} & \textbf{Tri.} \\\midrule
    Kullback-Leibler & $\kl(\distp~\|~\distq)$ & $t\log t$ & \xmark & \xmark\\
    Reverse-KL & $\kl(\distq~\|~\distp)$ & $-\log t$ & \xmark & \xmark \\
    Jensen-Shannon & $\jsd(\distp,\distq)\defeq\frac{1}{2}(\kl(\distp\|\distm) + \kl(\distq\|\distm))$ & $t\log t - (t+1)\log(\frac{t+1}{2})$ & \cmark & \xmark \\
    Squared Hellinger & $H^2(\distp, \distq)\defeq\frac{1}{2}\int (\sqrt{d\distp} - \sqrt{d\distq})^2$ & $(1-\sqrt{t})^2/2$ & \cmark & \xmark \\
    Total Variation & $\dtv(\distp, \distq)\defeq\sup_{E}|\distp(E) - \distq(E)|$ & $|t - 1|/2$ & \cmark & \cmark \\\bottomrule
    \end{tabular}
\end{table}

\section{Tradeoff between Fairness and Accuracy}
\label{sec:main}
In this section we take a slight detour from learning fair representations to first provide general results on the tradeoff between fairness and accuracy that applies to any fair classifiers. As we briefly discussed in Section~\ref{sec:gfair}, it is impossible to have imperfect predictor that is both statistically calibrated and verifies equalized odds when the base rates differ between two groups. On the other hand, while it has long been observed that demographic parity may eliminate perfect predictor~\citep{hardt2016equality}, and previous work has empirically verified that tradeoff exists between accuracy and demographic parity~\citep{calders2009building,kamiran2009classifying,zliobaite2015relation} on various datasets, so far a quantitative characterization on the exact tradeoff between accuracy and various notions of group fairness is still missing in the classification setting. In this section, we seek to answer the following intriguing and important question:
\begin{quote}
    \itshape
    In the setting of classification, what is the minimum error that any fair algorithm has to incur, and how does this error depend on the coupling between the target and the protected attribute?
\end{quote}
In what follows we shall first provide a simple example to illustrate this general tradeoff. This example will give readers a flavor the kind of impossibility result we are interested in obtaining. We then proceed to formally present a family of information-theoretic lower bounds on the accuracy that hold for \emph{all} algorithms, even if only approximate statistical parity is satisfied. We conclude this section by some discussions on the implications of our results.

\paragraph{A Simple Example}
As a warm-up, let us consider an example to showcase the potential tradeoff between statistical parity and accuracy in binary classification. But before our construction, it should be noted that the error $\err_{\dist}(\Ypred)$ bears an intrinsic lower bound for any classifier $\Ypred = h(X)$, i.e., the noise in the underlying data distribution $\dist$, e.g., the Bayes error rate. Hence to simplify our discussion, in this example we shall construct distributions such that the Bayes error rates are 0, i.e., for $a\in\{0, 1\}$, there exists a ground-truth labeling function $h_a^*$ such that $Y = h_a^*(X)$ on $\dist_a$. Realize that such simplification will only make it harder for us to prove lower bound on $\err_{\dist_a}$ since there exists classifiers that are perfect.
\begin{example}[A bijection between the target and the protected attribute]
    \label{exp:simple}
    For $a\in\{0, 1\}$, let the marginal distribution $X_\sharp\dist_a$ be a uniform distribution over $\{0, 1\}$. Let $Y = a$ be a constant. Hence by construction, on the joint distribution, we have $Y = A$ hold. Now for any fair predictor $\Ypred = h(X)$, the statistical parity asks $\Ypred$ to be independent of $A$.
    However, no matter what value $h(x)$ takes, we always have $|h(x)| + |h(x) - 1| \geq 1$. Hence for any predictor $h:\xxspace\to\{0,1\}$:
    \begin{align*}
        \err_{\dist_0}(h) + \err_{\dist_1}(h) &= \frac{1}{2}|h(0) - 0| + \frac{1}{2}|h(1) - 0| + \frac{1}{2}|h(0) - 1| + \frac{1}{2}|h(1) - 1| \\
        &\geq \frac{1}{2} + \frac{1}{2} \\
        &= 1.
    \end{align*}
    This shows that for any fair predictor $h$, the sum of the errors of $h$ on both groups has to be at least 1. On the other hand, there exists a trivial unfair algorithm that makes no error on both groups by also taking the protected attribute into consideration: $\forall x\in\{0, 1\}, h^*(x) = 0$ if $A = 0$ else $h^*(x) = 1$.
\end{example}

\subsection{An Accuracy Lower Bound for Fair Classifiers}
In this subsection we generalize the above simple example to general cases without making explicit assumptions on the underlying data generating distributions. Essentially, every prediction function induces the following Markov chain: 
\begin{equation*}
X\overset{g}{\longrightarrow} Z\overset{h}{\longrightarrow} \Ypred,    
\end{equation*}
where $g$ is the feature transformation, $h$ is the classifier on feature space, $Z$ is the feature and $\Ypred$ is the predicted target variable by $h\circ g$. Note that simple models, e.g., linear classifiers, are also included by specifying $g$ to be the identity map. With this notation, we first state the following theorem that quantifies an inherent tradeoff between fairness and accuracy.
\begin{theorem}
\label{thm:lowerbound}
Let $\Ypred = h(g(X))$ be a predictor. If $\Ypred$ satisfies demographic parity, then $\err_{\dist_0}(h\circ g) + \err_{\dist_1}(h\circ g) \geq \dbr(\dist_0, \dist_1)$.
\end{theorem}
\paragraph{Remark}
It is worth pointing out that Theorem~\ref{thm:lowerbound} holds for any representation function $g$ and classifier $h$, as long as the final predictor $\Ypred$ satisfies demographic parity. In particular, by restricting $g$ to be the identity function, we see that the lower bound also holds for any classifier that directly acts on the original input data. We choose the current presentation using a composition function $h\circ g$ only for consistency with the rest of the paper.

Next, $\dbr(\dist_0, \dist_1)$ is the difference of base rates across groups, and it achieves its maximum value of 1 iff there exists a bijection between $Y$ and $A$, e.g., $Y = A$. On the other hand, if $Y$ is independent of $A$, then $\dbr(\dist_0, \dist_1) = 0$ so the lower bound does not make any constraint on the joint error. Hence, the lower bound could also be understood as an uncertainty principle in fairness, stating in general that when the difference of base rates is large, then any fair algorithm has to incur a large error on at least one of the subgroups. Lastly, although Theorem~\ref{thm:lowerbound} asks for exact demographic parity, in Section~\ref{sec:representation} we shall extend the above theorem when only approximate demographic parity is met, via learning fair representations.

Note that from Example~\ref{exp:simple}, we can see this lower bound is tight, in the sense that there exist problem instances where the equality is verified. Second, Theorem~\ref{thm:lowerbound} applies to all possible feature transformation $g$ and predictor $h$. In particular, if we choose $g$ to be the identity map, then Theorem~\ref{thm:lowerbound} says that when the base rates differ, \emph{no algorithm} can achieve a small joint error on both groups, and it also recovers the previous observation that demographic parity can eliminate the perfect predictor~\citep{hardt2016equality}. Third, the lower bound in Theorem~\ref{thm:lowerbound} is insensitive to the marginal distribution of $A$, i.e., it treats the errors from both groups equally. As a comparison, let $\alpha\defeq \dist(A = 1)$, then $\err_\dist(h\circ g) = (1 - \alpha) \err_{\dist_0}(h\circ g) + \alpha\err_{\dist_1}(h\circ g)$. In this case $\err_\dist(h\circ g)$ could still be small even if the minority group suffers a large error. More formally, for the joint error $\err_\dist(h\circ g)$, we have the following corollary hold:
\begin{corollary}
\label{cor:joint}
    Let $\Ypred = h(g(X))$ be a predictor. If $\Ypred$ satisfies demographic parity, then the joint error has the following lower bound: $\err_{\dist}(\Ypred) \geq \Hzo(A)\cdot \dbr(\dist_0, \dist_1)$.
\end{corollary}
Compared with the lower bound in Theorem~\ref{thm:lowerbound}, the lower bound of the joint error in Corollary~\ref{cor:joint} additionally depends on the zero-one entropy of $A$. In particular, if the marginal distribution of $A$ is skewed, then $\Hzo(A)$ will be small, which means that fairness will not reduce the joint accuracy too much. This corollary further implies that when the demographic subgroups are imbalanced in the overall population, the joint accuracy is not an ideal metric to look at, since it may hide the potentially large drop in accuracy of the minority group. In particular, by the pigeonhole principle, the following corollary holds:
\begin{corollary}
If the predictor $\Ypred = h(g(X))$ satisfies demographic parity, then $\max\{\err_{\dist_0}(h\circ g), \err_{\dist_1}(h\circ g)\} \geq \dbr(\dist_0, \dist_1)/2$. 
\end{corollary}
In words, this means that for fair predictors in the demographic parity sense, at least one of the subgroups has to incur an error of at least $\dbr(\dist_0, \dist_1)/2$, which further emphasizes the fundamental role of the difference of base rates, $\dbr(\dist_0, \dist_1)$, in the tradeoff between fairness and accuracy.

\paragraph{Proofs of Theorem~\ref{thm:lowerbound}, Corollary~\ref{cor:joint}}
Before we present the proof, we first present a useful lemma that lower bounds the prediction error by the total variation distance.
\begin{restatable}{lemma}{errorlemma}
\label{lemma:dtverror}
    Let $\Ypred = h(X)$ be a predictor, then for $a\in\{0, 1\}$, $\dtv(\dist_a(Y), \dist_a(\Ypred))\leq \err_{\dist_a}(h)$.
\end{restatable}
\begin{proof}
    For $a\in\{0, 1\}$, because both $\dist_a(Y)$ and $\dist_a(\Ypred)$ are Bernoulli distributions, we have:
    \begin{align*}
        \dtv(\dist_a(Y), \dist_a(\Ypred)) &= |\dist_a(Y = 1) - \dist_a(h(X) = 1)| \\
        &= \left|\Exp_{\dist_a}[Y] - \Exp_{\dist_a}[h(X)]\right| \\
        &\leq \Exp_{\dist_a}\left[|Y - h(X)|\right] \\
        &= \err_{\dist_a}(h),
    \end{align*}
where the last equality holds because $\Pr_{\dist_a}(Y\neq\Ypred) = \Exp_{\dist_a}\left[|Y - h(X)|\right]$ when $Y,\Ypred\in\{0, 1\}$.
\end{proof}
Now we are ready to prove Theorem~\ref{thm:lowerbound}:
\begin{proof}[Proof of Theorem~\ref{thm:lowerbound}]
    First of all, we show that if $\Ypred = h(g(X))$ satisfies demographic parity, then:
    \begin{align*}
        \dtv(\dist_0(\Ypred), \dist_1(\Ypred)) &= \max\big\{|\dist_0(\Ypred = 0) - \dist_1(\Ypred = 0)|,~|\dist_0(\Ypred = 1) - \dist_1(\Ypred = 1)|\big\} \\
        &= |\dist_0(\Ypred = 1) - \dist_1(\Ypred = 1)| \\
        &= |\dist(\Ypred = 1\mid A = 0) - \dist(\Ypred = 1\mid A = 1)| = 0,
    \end{align*}
    where the last equality follows from the definition of demographic parity. Now from Table~\ref{tab:fdiv}, $\dtv(\cdot, \cdot)$ is symmetric and satisfies the triangle inequality, we have:
    \begin{align}
        \dtv(\dist_0(Y), \dist_1(Y)) &\leq \dtv(\dist_0(Y), \dist_0(\Ypred)) + \dtv(\dist_0(\Ypred), \dist_1(\Ypred)) + \dtv(\dist_1(\Ypred), \dist_1(Y)) \nonumber\\
        &= \dtv(\dist_0(Y), \dist_0(\Ypred)) + \dtv(\dist_1(\Ypred), \dist_1(Y)). 
        \label{equ:1}
    \end{align}
    The last step is to bound $\dtv(\dist_a(Y), \dist_a(\Ypred))$ in terms of $\err_{\dist_a}(h\circ g)$ for $a\in\{0, 1\}$ using Lemma~\ref{lemma:dtverror}:
    \begin{equation*}
        \dtv(\dist_0(Y), \dist_0(\Ypred)) \leq \err_{\dist_0}(h\circ g),\quad\dtv(\dist_1(Y), \dist_1(\Ypred)) \leq \err_{\dist_1}(h\circ g).
    \end{equation*}
    Combining the above two inequalities and~\eqref{equ:1} completes the proof.
\end{proof}
We now provide the proof of Corollary~\ref{cor:joint} on the lower bound of the joint error.
\begin{proof}[Proof of Corollary~\ref{cor:joint}]
    To simplify the notation used in the proof, define $\eps\defeq\err_{\dist}(\Ypred)$, $\eps_0\defeq\err_{\dist_0}(\Ypred)$ and $\eps_1\defeq\err_{\dist_1}(\Ypred)$. Let $\alpha\defeq\Pr_\dist(A = 0)$. By Theorem~\ref{thm:lowerbound}, we know that $\eps_0 + \eps_1 \geq \dbr(\dist_0,\dist_1)$. By definition of the joint error:
    \begin{align*}
        \eps &= \alpha\eps_0 + (1-\alpha)\eps_1 \\
        &\geq \alpha\eps_0 + (1-\alpha)(\dbr(\dist_0,\dist_1) - \eps_0) \\
        &= (1-\alpha)\dbr(\dist_0, \dist_1) + (2\alpha - 1)\eps_0.
    \end{align*}
    Similarly, we can also lower bound the joint error by:
    \begin{equation*}
        \eps \geq \alpha \dbr(\dist_0, \dist_1) + (1-2\alpha)\eps_1.
    \end{equation*}
    Now we discuss in two cases. If $\alpha\leq 1/2$, considering the second inequality yields:
    \begin{align*}
        \eps \geq \alpha \dbr(\dist_0, \dist_1) + (1-2\alpha)\eps_1 \geq \alpha \dbr(\dist_0, \dist_1).
    \end{align*}    
    If $\alpha > 1/2$, using the first inequality we have:
    \begin{align*}
        \eps \geq (1-\alpha)\dbr(\dist_0, \dist_1) + (2\alpha - 1)\eps_0 \geq (1-\alpha)\dbr(\dist_0, \dist_1).
    \end{align*}    
    Combining the above two cases leads to:
    \begin{equation*}
        \eps \geq \min\{\alpha, 1-\alpha\}\cdot \dbr(\dist_0, \dist_1) = \Hzo(A)\cdot \dbr(\dist_0, \dist_1),
    \end{equation*}
    completing the proof.
\end{proof}

It is not hard to show that our lower bound in Theorem~\ref{thm:lowerbound} is tight. To see this, consider the case $A = Y$, where the lower bound achieves its maximum value of 1. Now consider a constant predictor $\Ypred \equiv 1$ or $\Ypred\equiv 0$, which clearly satisfies demographic parity by definition. But in this case either $\err_{\dist_0}(h\circ g) = 1, \err_{\dist_1}(h\circ g) = 0$ or $\err_{\dist_0}(h\circ g) = 0, \err_{\dist_1}(h\circ g) = 1$, hence $\err_{\dist_0}(h\circ g) + \err_{\dist_1}(h\circ g) \equiv 1$, achieving the lower bound. 

To conclude this section, we point out that the choice of total variation in the lower bound is not unique. As we will see shortly in Section~\ref{sec:representation}, similar lower bounds could be attained using specific choices of the general $f$-divergence with some desired properties. 

\subsection{Extension to Multiple Subgroups under Multi-class Classification}
\label{sec:multi}
\begin{figure}[tb]
    \centering
    \includegraphics[width=0.8\linewidth]{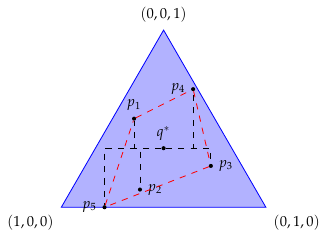}
    \caption{An example of the TV-Barycenter problem in~\eqref{equ:gopt} where $n=5$ and $m=3$. The optimal solution vector $q^*$ corresponds to the barycenter in the convex hull of $\{p_1,\ldots, p_5\}$ that minimizes the sum of $\ell_1$ distances between $p^*$ and $p_i$, $i\in [5]$.}
    \label{fig:my_label}
\end{figure}
The analytical tradeoff lower bound in Theorem~\ref{thm:lowerbound} mainly works for the setting of binary classification $(|\yyspace| = 2)$ with binary protected attribute $(|\aaspace| = 2)$. Hence it is natural to ask whether such lower bounds also exist in the general setting where the target variable $Y$ is a discrete random variable that takes $m \geq 2$ values and the protected attribute $A$ is also a categorical random variable that takes $n \geq 2$ different values. In what follows we shall provide an extension of Theorem~\ref{thm:lowerbound} to this general setting, although in this case we can no longer obtain an analytical characterization of the lower bound. However, as we shall see shortly, the exact lower bound could still be efficiently computed by solving a linear program. 

We first introduce some additional notation that will be used in this section. For a positive integer $K\in\Nat^*$, we use $[K]$ to denote the set $\{1, \ldots, K\}$. We use $\Delta_K$ to denote the $K$-dimensional probability simplex, i.e., $p\in\Delta^K$ if $p\in\RR_+^K$ and $\sum_{i\in [K]} p_i = 1$. For each subgroup $A = i, i\in [n]$, the corresponding marginal distribution of the target label $Y$ could then be described by a $m$-dimensional vector $p_i\in\Delta_m$, i.e., $\dist_i(Y) = p_i$. With these notation, we could then formally establish the following optimization problem:
\begin{equation}
\label{equ:gopt}
    \begin{aligned}
\text{(TV-Barycenter)}:\qquad    &\min_{q} && \frac{1}{2}\sum_{i=1}^n \|q - p_i\|_1 = \frac{1}{2}\sum_{i=1}^n \sum_{j=1}^m |(p_i)_j - q_j| \\
    &\text{subject to} && q\in\Delta_m: q \geq 0,~\sum_{j=1}^m q_j = 1
    \end{aligned}
\end{equation}
We use $\opt(\{p_i\}_{i\in[n]})$ to denote the optimal value of the optimization problem in~\eqref{equ:gopt}. As the name suggests, the above optimization problem computes the barycenter (under the $\ell_1$ distance) $q^*$ of the set of vectors $\{p_i\}_{i\in [n]}$, where each $p_i\in\Delta_m$ corresponds to the marginal label distribution of $Y$ within the group $A = i$. Clearly, the TV-Barycenter problem is a linear program, and hence its optimal solution could be efficiently computed in polynomial time.

We now state the extension of Theorem~\ref{thm:lowerbound} using the optimal solution to the TV-Barycenter problem:
\begin{theorem}
\label{thm:general}
Define $p_i$ to be the probability mass vector of $\dist_i(Y)$: $\forall i\in[n], j\in[m], (p_i)_j = \Pr_{\dist_i}(Y = j)$. Let $\Ypred = h(X)$ be a predictor. If $\Ypred$ satisfies demographic parity, then $\sum_{i=1}^n \err_{\dist_i}(h) \geq \opt(\{p_i\}_{i\in[n]})$.
\end{theorem}
\paragraph{Remark} To see that Theorem~\ref{thm:general} is indeed a generalization of Theorem~\ref{thm:lowerbound}, note that when there are only two groups, i.e., $n = 2$, we can readily read off the optimal solution $q^*$ as: $q^* = (p_1 + p_2) / 2$ by realizing that the objective function is fully decomposable, and $\opt(\{p_i\}_{i\in[n]})$ is
\begin{equation*}
    \frac{1}{2}\left\|p_1 - \frac{p_1 + p_2}{2}\right\|_1 + \frac{1}{2}\left\|p_2 - \frac{p_1 + p_2}{2}\right\|_1 = \frac{1}{2}\|p_1 - p_2\|_1 = \dtv(\dist_1(Y), \dist_2(Y)).
\end{equation*}
Furthermore, when $m = 2$, $\dtv(\dist_1(Y), \dist_2(Y)) = \dbr(\dist_1, \dist_2)$. In general when $n > 2$, we cannot expect to have an analytic solution of $\opt(\{p_i\}_{i\in[n]})$, but nevertheless it can be computed efficiently by solving a linear program. From this perspective, Theorem~\ref{thm:general} builds an equivalent connection between the tradeoff problem in fairness and the barycenter problem in TV-distance. 
\begin{proof}[Proof of Theorem~\ref{thm:general}]
    For $\Ypred = h(X)$, let $C^{(i)}\in\RR_+^{m\times m}$ be the confusion matrix between $Y$ and $\Ypred$ under $\dist_i$, $\forall i\in[n]$. By definition of the confusion matrix, we have $C^{(i)}_{jj'} = \Pr_{\dist_i}(Y = j, \Ypred = j')$. Hence, 
    \begin{equation*}
        \Pr_{\dist_i}(Y\neq \Ypred) = 1 - \sum_{j=1}^m \Pr_{\dist_i}(Y = \Ypred = j) = 1 - \tr(C^{(i)}).
    \end{equation*}
    On the other hand, consider $\dtv(\dist_i(Y), \dist_i(\Ypred))$, we have
    \begin{align*}
        \dtv(\dist_i(Y), \dist_i(\Ypred)) &= \frac{1}{2}\sum_{j=1}^m \left|\Pr_{\dist_i}(Y = j) - \Pr_{\dist_i}(\Ypred = j)\right| \\
        &= \frac{1}{2}\sum_{j=1}^m \left|\sum_{j'\in [m]} C^{(i)}_{jj'} - \sum_{j'\in [m]}C^{(i)}_{j' j}\right| = \frac{1}{2}\sum_{j=1}^m \left|\sum_{j'\neq j} (C^{(i)}_{jj'} - C^{(i)}_{j' j})\right| \\
        &\leq \frac{1}{2}\sum_{j=1}^m\sum_{j'\neq j}C^{(i)}_{jj'} + \frac{1}{2}\sum_{j=1}^m\sum_{j'\neq j}C^{(i)}_{j'j} = \sum_{j\neq j'}C^{(i)}_{jj'} \\
        &= 1 - \tr(C^{(i)}) = \Pr_{\dist_i}(Y\neq \Ypred).
    \end{align*}
    Next, since $\Ypred$ satisfies demographic parity, so $\Ypred\perp A$, which means $\nu\defeq \dist_1(\Ypred) = \cdots = \dist_n(\Ypred)$. Combining the above two arguments together yields
    \begin{align*}
        \inf_{\Ypred = h(X)}\sum_{i\in [n]}\err_{\dist_i}(\Ypred) &= \inf_{\Ypred = h(X)}\sum_{i\in [n]}\Pr_{\dist_i}(Y\neq \Ypred) \\
        &\geq \inf_{\Ypred = h(X)}\sum_{i\in [n]} \dtv(\dist_i(Y), \dist_i(\Ypred)) \\
        &= \inf_{\Ypred = h(X)}\sum_{i\in [n]} \dtv(\dist_i(Y), \nu) \\
        &\geq \inf_{\nu\in\Delta_m}\sum_{i\in [n]} \dtv(\dist_i(Y), \nu) \\
        &= \min_{q\in\Delta_m}\frac{1}{2}\sum_{i=1}^n \|q - p_i\|_1 \\
        &= \opt(\{p_i\}_{i\in[n]}),
    \end{align*}
completing the proof.
\end{proof}

\section{An Optimal Fair Classifier}
\begin{algorithm}[tb]
\caption{Optimal fair classifier}
\label{alg:fairopt}
\begin{algorithmic}[1]
\Require Oracle access to $h_0^*$ and $h_1^*$, the Bayes optimal classifiers over $\dist_0$ and $\dist_1$
\Ensure A randomized optimal fair classifier $\fairopt:\xxspace\times\aaspace\to\yyspace$
\State  Compute $\alpha\defeq \Pr_{\dist_0}(Y = 1)$ and $\beta\defeq \Pr_{\dist_1}(Y=1)$. Without loss of generality assume $\alpha\geq \beta$
\State  For $(x, a)$, randomly sample $s\sim U(0, 1)$, the uniform distribution between $(0,1)$
\State  Construct $\fairopt(x, a)$ as
        \begin{equation}
        \label{equ:fairopt}
            \fairopt(x, a)\defeq\begin{cases}
                a = 0: \begin{cases}
                0   & \text{If } h_0^*(x) = 0 \text{ or } h_0^*(x) = 1 \text{ and } s > \frac{\alpha + \beta}{2\alpha}\\
                1   & \text{If } h_0^*(x) = 1 \text{ and } s \leq \frac{\alpha + \beta}{2\alpha}
                \end{cases} \\
                a = 1: \begin{cases}
                0   & \text{If } h_1^*(x) = 0 \text{ and } s > \frac{\alpha - \beta}{2(1 - \beta)}\\
                1   & \text{If } h_1^*(x) = 1 \text{ or } h_1^*(x) = 0 \text{ and } s \leq \frac{\alpha - \beta}{2(1 - \beta)}
                \end{cases}
            \end{cases}
        \end{equation}
\Return $\fairopt$
\end{algorithmic}
\end{algorithm}

Theorem~\ref{thm:lowerbound} provides an information-theoretic lower bound on the sum of group-wise errors for any fair classifiers. From the proof of Theorem~\ref{thm:lowerbound}, it is clear that the same lower bound also holds for fair classifiers that can have explicit access to the protected attribute $A$. To see this, consider a special case where the input $X$ contains a redundant attribute that is a synonym of the protected attribute $A$. Since Theorem~\ref{thm:lowerbound} holds for any distribution $\mu$ over the triplet $(X, A, Y)$, this simple observation implies that our lower bound also holds for fair classifiers that take the protected attribute $A$ as an input explicitly. 

Although in the last section we briefly mention the tightness of Theorem~\ref{thm:lowerbound} by constructing problem instances and fair classifiers where the equality verifies, it is still unclear whether it is possible to construct an algorithm such that:
\begin{enumerate}
    \item   For any distribution $\dist$ over $(X, A, Y)$, the algorithm returns a (possibly randomized) fair classifier $\fairopt$. 
    \item   The returned fair classifier $\fairopt$ is optimal, in the sense that it verifies the lower bound in Theorem~\ref{thm:lowerbound}: $\err_{\dist_0}(\fairopt) + \err_{\dist_1}(\fairopt) = \dbr(\dist_0, \dist_1)$.
\end{enumerate}

It should be noted that in general it is relatively easy to construct a trivial classifier that is fair in the demographic parity sense. For example, any constant classifier that always outputs 0 or 1 is always fair for any distribution $\dist$ over $(X, A, Y)$, but this classifier is not optimal in the sense of achieving the best possible accuracy. On the other hand, the problem of learning an optimal fair classifier is at least as hard as learning a Bayes optimal classifier, which we formally define as follows.
\begin{definition}[Bayes Optimal Classifier]
Given random variables $(X,A)$ and a target variable $Y$, the Bayes optimal classifier is $h^*\defeq \argmin_{h(x,a)}\Pr(Y\neq h(X, A))$ with $h^*(x, a) = 1$ iff $\Exp[Y\mid X=x, A=a]\geq 1/2$ otherwise 0. We also use $h_a^*(\cdot)$ to denote the restriction of $h^*$ on $A = a$, respectively, i.e., $h_a^*(\cdot)\defeq h^*(\cdot, a)$.
\end{definition}
Namely, $h_a^*(\cdot)$ is the corresponding Bayes optimal classifier on the group $A = a$ for $a\in\{0, 1\}$. With the definition of Bayes optimal classifier, we can formally argue that learning an optimal fair classifier is at least as hard as learning a Bayes optimal classifier by the following reduction.

\paragraph{A Reduction}
Given a distribution $\dist'$ over $\xxspace\times\yyspace$, let $h'(\cdot)$ be the Bayes optimal classifier over $\dist'$. We can create a problem instance of learning the optimal fair classifier by constructing a distribution $\dist$ over $(X, A, Y)$ with $\dist_{A=0} = \dist_{A=1} = \dist'$. Now because $\dist_{A=0} = \dist_{A=1} = \dist'$, it is clear to see that the classifier $h^*(\cdot, \cdot)$ where $h^*(\cdot, 0) = h^*(\cdot, 1) = h'^*(\cdot)$ satisfies demographic parity. Furthermore, the optimality of $h'^*$ over $\dist'$ implies the optimality of $h^*$ over $\mu$ as well. This shows that an oracle call to the problem of learning the optimal fair classifier over $\dist$ can be used to solve the problem of learning a Bayes optimal classifier over $\dist'$, by a restriction of the returned $h^*$ to either $h_0^*$ or $h_1^*$.

Surprisingly, the other direction is also true. More specifically, in what follows we shall present an algorithm to construct a randomized classifier that is both fair and optimal, given oracle access to $h_0^*$ and $h_1^*$, i.e., the group-wise Bayes optimal classifiers. We list the algorithm in Algorithm~\ref{alg:fairopt} and the corresponding decision diagram of $\fairopt$ in Figure~\ref{fig:fairopt}. 
\begin{figure}[tb]
    \centering
    \includegraphics[width=0.8\linewidth]{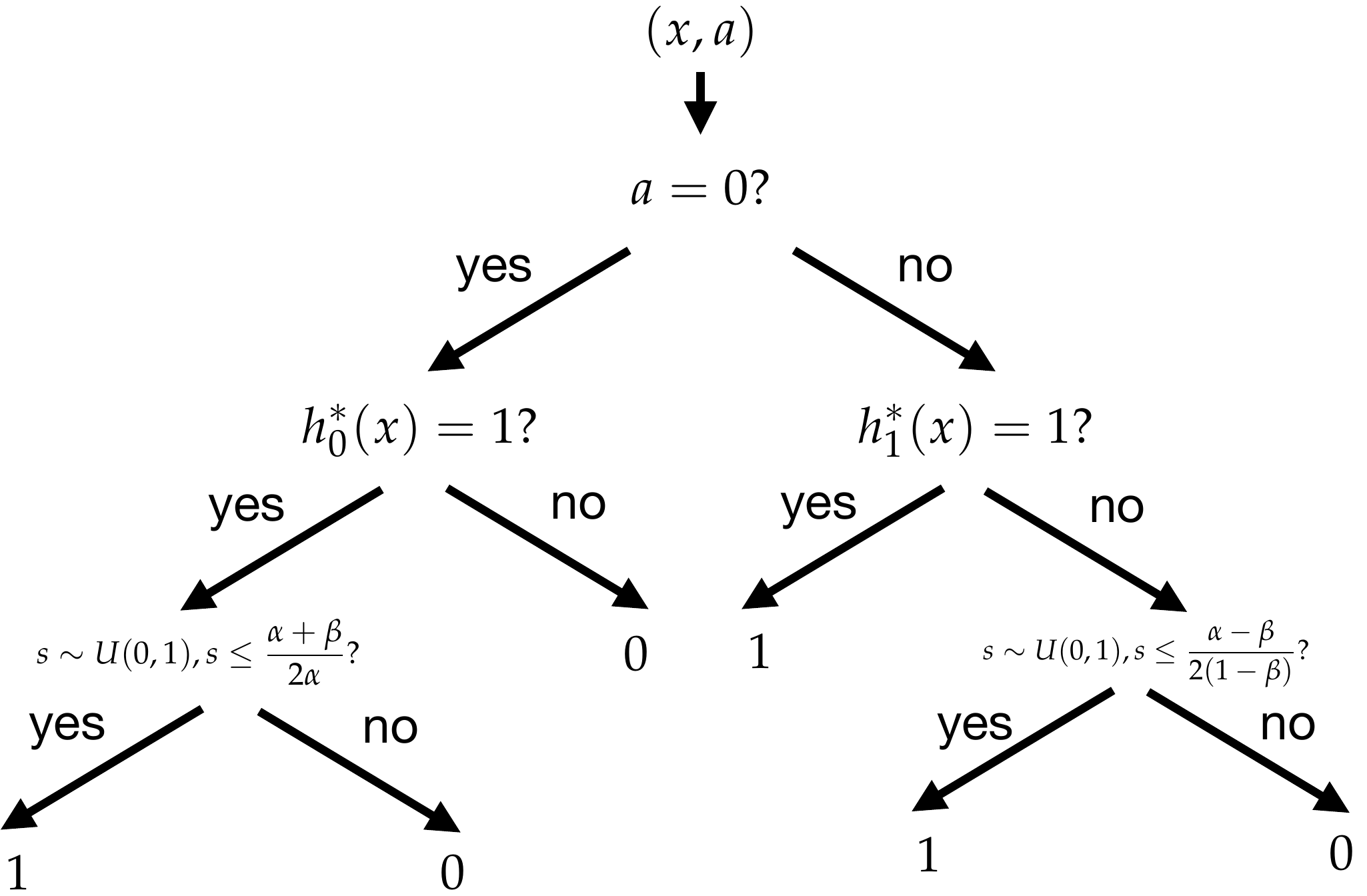}
    \caption{The decision tree diagram of the (randomized) optimal fair classifier $\fairopt$. The optimal fair classifier $\fairopt$ has oracle access to $h_0^*$ and $h_1^*$, the (possibly unfair) Bayes optimal classifier on $\dist_0$ and $\dist_1$, respectively. $\alpha, \beta$ are the base rates over $\dist_0$ and $\dist_1$, i.e, $\alpha\defeq \Pr_{\dist_0}(Y=1)$ and $\beta\defeq \Pr_{\dist_1}(Y=1)$.}
    \label{fig:fairopt}
\end{figure}
\begin{theorem}
\label{thm:fairopt}
    For any distribution $\dist$ over $(X, A, Y)$ such that $Y_{A = 0} = h_0^*(X)$ and $Y_{A=1} = h_1^*(X)$, the classifier $\fairopt$ constructed by Algorithm~\ref{alg:fairopt} satisfies demographic parity and is optimal, i.e., $\err_{\dist_0}(\fairopt) + \err_{\dist_1}(\fairopt) = \dbr(\dist_0, \dist_1)$.
\end{theorem}
Before we present the proof of Theorem~\ref{thm:fairopt}, we first briefly discuss one implication of its assumption that $Y_{A = 0} = h_0^*(X)$ and $Y_{A=1} = h_1^*(X)$. Essentially, this assumption says that there exists a perfect but potentially unfair classifier over $\dist$. Hence in this case $\err_{\dist_0} + \err_{\dist_1}$ exactly corresponds to the price paid by enforcing demographic parity. In what follows we provide the proof for this theorem. 
\begin{proof}[Proof of Theorem~\ref{thm:fairopt}]
    We first show that $\fairopt$ is fair in the demographic parity sense. Let $\Ypred = \fairopt(X, A)$, $\alpha\defeq \Pr(Y = 1\mid A = 0)$ and $\beta\defeq \Pr(Y = 1 \mid A = 1)$. Without loss of generality, we assume $\alpha\geq\beta$. 
    
    For $A = 0$, consider the probability $\Pr_{\dist_0}(\Ypred = 1)$. Note that by construction in Algorithm~\ref{alg:fairopt}, $\fairopt(X, 0) = 0$ whenever $h_0^*(X) = 0$, so
    \begin{align*}
        \Pr_{\dist_0}(\Ypred = 1) &= \Pr_{\dist_0}\left(h_0^*(X) = 1, S \leq \frac{\alpha + \beta}{2\alpha}\right) = \Pr_{\dist_0}\left(h_0^*(X) = 1\right)\cdot \Pr\left(S \leq \frac{\alpha + \beta}{2\alpha}\right) \\
        &= \Pr_{\dist_0}(Y = 1) \cdot \frac{\alpha + \beta}{2\alpha} = \alpha\cdot \frac{\alpha + \beta}{2\alpha} = \frac{\alpha + \beta}{2}.
    \end{align*}
    Similarly, for $A = 1$, recall that by construction, $\fairopt(X, 1) = 1$ whenever $h_1^*(X) = 1$, so
    \begin{align*}
        \Pr_{\dist_1}(\Ypred = 1) &= \Pr_{\dist_1}\left(h_1^*(X) = 1\vee \left(h_1^*(X) = 0 \wedge S \leq \frac{\alpha - \beta}{2(1 - \beta)}\right)\right) \\
        &= \Pr_{\dist_1}\left(h_1^*(X) = 1\right) + \Pr\left(h_1^*(X) = 0\right)\cdot\Pr\left(S \leq \frac{\alpha - \beta}{2(1 - \beta)}\right) \\
        &= \Pr_{\dist_1}(Y = 1) + \Pr_{\dist_1}(Y = 0)\cdot  \frac{\alpha - \beta}{2(1 - \beta)} \\
        &= \beta + (1 - \beta)\cdot\frac{\alpha - \beta}{2(1 - \beta)}  = \frac{\alpha + \beta}{2},
    \end{align*}
    where in the proof above we use the fact that $S\sim U(0, 1)$ is drawn independently of $X$. This shows that $\dist_0(\Ypred) = \dist_1(\Ypred)$ so $\Ypred = \fairopt(X, A)$ is fair. 
    
    Next, we prove that $\fairopt$ is optimal. For $A = 0$, 
    \begin{equation*}
        \err_{\dist_0}(\Ypred) = \Pr(\fairopt(X, 0)\neq Y\mid A = 0) = \Pr(\fairopt(X, 0) \neq h_0^*(X)\mid A= 0).
    \end{equation*}
    However, due to the construction of $\fairopt$ in Algorithm~\ref{alg:fairopt}, $\fairopt(X, 0)\neq h_0^*(X)$ could only happen if $h_0^*(X) = 1$ while $\fairopt(X, 0) = 0$, so
    \begin{align*}
        \err_{\dist_0}(\Ypred) &= \Pr(\fairopt(X, 0) \neq h_0^*(X)\mid A = 0) \\
                                &= \sum_{x}\Pr(\fairopt(x, 0)\neq h_0^*(x)\mid A = 0) \\
                                &= \sum_{x}\Pr(\fairopt(x, 0) = 0, h_0^*(x) = 1\mid A = 0) \\
                                &= \sum_{x}\Pr(\fairopt(x, 0) = 0\mid h_0^*(x) = 1, A = 0)\cdot \Pr(h_0^*(x) = 1\mid A =0) \\
                                &= \sum_{x}\Pr\left(S > \frac{\alpha + \beta}{2\alpha}\right)\cdot \Pr(h_0^*(x) = 1\mid A = 0) \\
                                &= \left(1 - \frac{\alpha + \beta}{2\alpha}\right)\cdot \sum_x \Pr(h_0^*(x) = 1\mid A = 0)\\
                                &= \left(1 - \frac{\alpha + \beta}{2\alpha}\right)\cdot \alpha \\
                                &= \frac{\alpha - \beta}{2}.
    \end{align*}
    Similarly, for $A = 1$, 
    \begin{equation*}
        \err_{\dist_1}(\Ypred) = \Pr(\fairopt(X, 1)\neq Y\mid A = 1) = \Pr(\fairopt(X, 1) \neq h_1^*(X)\mid A = 1).
    \end{equation*}
    Again, due to the construction of $\fairopt$ in Algorithm~\ref{alg:fairopt}, $\fairopt(X, 1)\neq h_1^*(X)$ could only happen if $h_1^*(X) = 0$ while $\fairopt(X, 0) = 1$, so
    \begin{align*}
        \err_{\dist_0}(\Ypred) &= \Pr(\fairopt(X, 1) \neq h_0^*(X)\mid A = 1) \\
                                &= \sum_{x}\Pr(\fairopt(x, 1)\neq h_1^*(x)\mid A = 1) \\
                                &= \sum_{x}\Pr(\fairopt(x, 1) = 1, h_1^*(x) = 0\mid A = 1) \\
                                &= \sum_{x}\Pr(\fairopt(x, 1) = 1\mid h_1^*(x) = 0, A = 1)\cdot \Pr(h_1^*(x) = 0\mid A =1) \\
                                &= \sum_{x}\Pr\left(S \leq \frac{\alpha - \beta}{2(1 - \beta)}\right)\cdot \Pr(h_1^*(x) = 0\mid A = 1) \\
                                &= \frac{\alpha - \beta}{2(1 - \beta)}\cdot \sum_x \Pr(h_1^*(x) = 0\mid A = 1)\\
                                &= \frac{\alpha - \beta}{2(1 - \beta)}\cdot (1 - \beta) \\
                                &= \frac{\alpha - \beta}{2}.
    \end{align*}
    Combining the two cases above shows that $\err_{\dist_0}(\fairopt) + \err_{\dist_1}(\fairopt) = \alpha - \beta = \dbr(\dist_0, \dist_1)$, which completes the proof.
\end{proof}
As a simple corollary of Theorem~\ref{thm:fairopt}, we can now strengthen Theorem~\ref{thm:lowerbound} as follows:
\begin{corollary}
\label{cor:opt}
Let $\Ypred = h(X, A)$ be a predictor that satisfies demographic parity, then $\inf_{\Ypred}\err_{\dist_0}(h) + \err_{\dist_1}(h) = \dbr(\dist_0, \dist_1)$.
\end{corollary}
Note that one difference between Corollary~\ref{cor:opt} and Theorem~\ref{thm:lowerbound} is that the fair predictor in Corollary~\ref{cor:opt} is allowed to have explicit access to the protected attribute $A$ during decision making. This is necessary for our construction of the optimal fair classifier in Algorithm~\ref{alg:fairopt}. It remains an open question whether it is possible to construct an optimal fair classifier to achieve the lower bound in Theorem~\ref{thm:lowerbound} without having explicit access to the protected attribute, as in many practical applications of high stakes the automated decision making process is regulated to not directly use the protected attribute $A$ (\citet{gdpr2021}, Article 22 Paragraph 4).

\section{Approximate Fairness via Learning Fair Representations}
\label{sec:representation}

In the last section we show that there is an inherent tradeoff between fairness and accuracy when a predictor \emph{exactly} satisfies demographic parity. In practice we may not be able to achieve demographic parity precisely. Instead, a line of recent algorithms~\citep{edwards2015censoring,louizos2015variational,beutel2017data,adel2019one,zhao2020conditional} build an adversarial discriminator that takes as input the feature vector $Z = g(X)$, and the goal is to learn fair representations such that it is hard for the adversarial discriminator to infer the group membership from $Z$, typically by solving a minimax objective between the feature encoder and the adversarial discriminator~\citep{edwards2015censoring,beutel2017data,zhang2018mitigating,zhao2020conditional,chi2021understanding}.

In these applications, due to the limit on the capacity of the adversarial discriminator, only approximate demographic parity can be achieved in practice. Hence it is natural to ask what is the tradeoff between fair representations and accuracy in this scenario? In this section we shall answer this question by generalizing our previous analysis with $f$-divergence to prove a family of lower bounds on the joint target prediction error. Our results also show how approximate DP helps to reconcile but not remove the tradeoff between fairness and accuracy. Before we state and prove the main results in this section, we first introduce the following lemma by~\citet{liese2006divergences} as a generalization of the data processing inequality for $f$-divergence:

\begin{lemma}[\citet{liese2006divergences}]
\label{lemma:kernel}
    Let $\Delta(\zzspace)$ be the space of all probability distributions over $\zzspace$. Then for any $f$-divergence
    $D_f(\cdot~\|~\cdot)$, any stochastic kernel $\kappa:\xxspace\to\Delta(\zzspace)$, and any distributions $\distp$ and $\distq$ over $\xxspace$, $D_f(\kappa\distp~\|~\kappa\distq)\leq D_f(\distp~\|~\distq)$.
\end{lemma}

Roughly speaking, Lemma~\ref{lemma:kernel} says that data processing cannot increase discriminating information. Define $\djs(\distp, \distq)\defeq \sqrt{\jsd(\distp, \distq)}$ and $H(\distp, \distq)\defeq \sqrt{H^2(\distp, \distq)}$. It is well-known in information theory that both $\djs(\cdot,\cdot)$ and $H(\cdot,\cdot)$ form a bounded distance metric over the space of probability distributions~\citep[Chapter 4]{wu2017lecture}. Realize that $\dtv(\cdot,\cdot)$, $H^2(\cdot,\cdot)$ and $\jsd(\cdot,\cdot)$ are all $f$-divergence. The following corollary holds:
\begin{corollary}
\label{cor:fdiv}
    Let $h:\zzspace\to\yyspace$ be a classifier, and $g_\sharp\dist_a$ be the pushforward distribution of $\dist_a$ by $g$, $\forall a\in\{0, 1\}$. Let $\Ypred = h(g(X))$ be the predictor, then all the following inequalities hold:
    \begin{enumerate}
        \item   $\dtv(\dist_0(\Ypred), \dist_1(\Ypred))\leq \dtv(g_\sharp\dist_0, g_\sharp\dist_1)$
        \item   $H(\dist_0(\Ypred), \dist_1(\Ypred))\leq H(g_\sharp\dist_0, g_\sharp\dist_1)$
        \item   $\djs(\dist_0(\Ypred), \dist_1(\Ypred))\leq \djs(g_\sharp\dist_0, g_\sharp\dist_1)$
    \end{enumerate}
\end{corollary}
Now we are ready to present the following main theorem of this section:
\begin{restatable}{theorem}{lowerbounds}
\label{thm:all}
Let $\Ypred = h(g(X))$ be the predictor where $h:\zzspace\to\{0,1\}$ is any classifier acting on the feature space. Assume $\djs(g_\sharp\dist_0, g_\sharp\dist_1)\leq\djs(\dist_0(Y), \dist_1(Y))$ and $H(g_\sharp\dist_0, g_\sharp\dist_1)\leq H(\dist_0(Y), \dist_1(Y))$, then the following three inequalities hold:
\begin{enumerate}
    \item   Total variation lower bound: 
        \begin{align*}
            \err_{\dist_0}(h\circ g) + \err_{\dist_1}(h\circ g) \geq \dtv(\dist_0(Y), \dist_1(Y)) - \dtv(g_\sharp\dist_0, g_\sharp\dist_1).
        \end{align*}
    \item   Jensen-Shannon lower bound: 
        \begin{align*}
            \err_{\dist_0}(h\circ g) + \err_{\dist_1}(h\circ g) \geq \big(\djs(\dist_0(Y), \dist_1(Y)) - \djs(g_\sharp\dist_0, g_\sharp\dist_1)\big)^2 / 2. 
        \end{align*}
    \item   Hellinger lower bound: 
        \begin{align*}
            \err_{\dist_0}(h\circ g) + \err_{\dist_1}(h\circ g) \geq \big(H(\dist_0(Y), \dist_1(Y)) - H(g_\sharp\dist_0, g_\sharp\dist_1)\big)^2 / 2.
        \end{align*}
\end{enumerate}
\end{restatable}
\paragraph{Remark}
All the three lower bounds in Theorem~\ref{thm:all} imply a tradeoff between the joint error across demographic subgroups and learning group-invariant feature representations. When $g_\sharp\dist_0 = g_\sharp\dist_1$, due to the data-processing principle, we also have $\dist_0(\Ypred) = \dist_1(\Ypred)$, and all three lower bounds get larger. In this case, we have 
\begin{align*}
    \max\left\{\dtv(\dist_0(Y), \dist_1(Y)), \frac{1}{2}\djs^2(\dist_0(Y), \dist_1(Y)), \frac{1}{2}H^2(\dist_0(Y), \dist_1(Y))\right\} =&~\dtv(\dist_0(Y), \dist_1(Y)) \\
     =&~\dbr(\dist_0, \dist_1),
\end{align*}
and this reduces to Theorem~\ref{thm:lowerbound}. We now present the proof for Theorem~\ref{thm:all}.
\begin{proof}[Proof of Theorem~\ref{thm:all}]
    We prove the three inequalities respectively. The total variation lower bound follows the same idea as the proof of Theorem~\ref{thm:lowerbound} and the inequality $\dtv(\dist_0(\Ypred), \dist_1(\Ypred))\leq \dtv(g_\sharp\dist_0, g_\sharp\dist_1)$ from Corollary~\ref{cor:fdiv}. To prove the Jensen-Shannon lower bound, realize that $\djs(\cdot,\cdot)$ is a distance metric over probability distributions. Combining with the inequality $\djs(\dist_0(\Ypred), \dist_1(\Ypred))\leq \djs(g_\sharp\dist_0, g_\sharp\dist_1)$ from Corollary~\ref{cor:fdiv}, we have:
    \begin{equation*}
        \djs(\dist_0(Y), \dist_1(Y)) \leq \djs(\dist_0(Y), \dist_0(\Ypred)) + \djs(g_\sharp\dist_0, g_\sharp\dist_1) + \djs(\dist_1(\Ypred), \dist_1(Y)).
    \end{equation*}
    Now by Lin's lemma~\citep[Theorem 3]{lin1991divergence}, for any two distributions $\distp$ and $\distq$, we have $\djs^2(\distp, \distq)\leq\dtv(\distp, \distq)$. Combine Lin's lemma with Lemma~\ref{lemma:dtverror}, we get the following lower bound:
    \begin{equation*}
        \sqrt{\err_{\dist_0}(h\circ g)} + \sqrt{\err_{\dist_1}(h\circ g)} \geq \djs(\dist_0(Y), \dist_1(Y)) - \djs(g_\sharp\dist_0, g_\sharp\dist_1).
    \end{equation*}
    Apply the AM-GM inequality, we can further bound the L.H.S.\ by
    \begin{equation*}
        \sqrt{2\big(\err_{\dist_0}(h\circ g) + \err_{\dist_1}(h\circ g)\big)}\geq \sqrt{\err_{\dist_0}(h\circ g)} + \sqrt{\err_{\dist_1}(h\circ g)}.
    \end{equation*}
    Under the assumption that $\djs(g_\sharp\dist_0, g_\sharp\dist_1)\leq\djs(\dist_0(Y), \dist_1(Y))$, taking a square at both sides then completes the proof for the second inequality. The proof for Hellinger's lower bound follows exactly as the one for Jensen-Shannon's lower bound, except that instead of Lin's lemma, we need to use the fact that $H^2(\distp,\distq)\leq \dtv(\distp,\distq)\leq \sqrt{2}H(\distp,\distq)$, $\forall\distp, \distq$.
\end{proof}
As a simple corollary of Theorem~\ref{thm:all}, the following result shows how approximate DP (in terms of the DP gap) helps to reconcile the tradeoff between fairness and accuracy, by controlling the divergence between different groups of representations:
\begin{corollary}
\label{cor:gap}
Let $g:\xxspace\to\zzspace$ be a feature transformation. If $\dtv(g_\sharp\dist_0, g_\sharp\dist_1)\leq\epsilon$, then for any classifier $h:\zzspace\to\yyspace$, the DP gap $\dpgap(h\circ g)\leq \epsilon$, and $\err_{\dist_0}(h\circ g) + \err_{\dist_1}(h\circ g) \geq \dbr(\dist_0, \dist_1) - \epsilon$.
\end{corollary}
In a sense Corollary~\ref{cor:gap} means that in order to lower the joint error, the DP gap of the predictor cannot be too small. Of course, since the above inequality is a lower bound, it only serves as a necessary condition for small joint error. Hence an interesting question would be to ask whether it is possible to have a sufficient condition that guarantees a small joint error such that the DP gap of the predictor is no larger than that of the perfect predictor, i.e., $\dbr(\dist_0, \dist_1)$. 

\section{Fair Representations Lead to Accuracy Parity}
In the previous sections we prove a family of information-theoretic lower bounds that demonstrate an inherent tradeoff between fair representations and joint error across groups. A natural question to ask then, is, what kind of parity can fair representations bring us? To complement our negative results, in this section we show that learning group-invariant representations help to reduce discrepancy of errors across groups. 

First of all, since we work under the stochastic setting where $\dist_a$ is a joint distribution over $X$ and $Y$ conditioned on $A = a$, then any function mapping $h:\xxspace\to\yyspace$ will inevitably incur an error due to the noise existed in the distribution $\dist_a$. In the case of binary classification, this error is also known as the Bayes error. Formally, for $a\in\{0, 1\}$, let $h_a^*:\xxspace\to\yyspace$ be the Bayes optimal classifier on $\dist_a$. Now define the noise of distribution $\dist_a$ (the Bayes error on $\dist_a$) to be $n_{\dist_a}\defeq \Pr_{\dist_a}(Y\neq h_a^*(X))   = \Exp_{\dist_a}[|Y - h_a^*(X)|]$. We are now ready to present the following theorem:
\begin{theorem}[Error Decomposition Theorem]
    For any hypothesis $\HH\ni h:\xxspace\to\yyspace$, the following inequality holds:
    \begin{align*}
        \left|\err_{\dist_0}(h) - \err_{\dist_1}(h)\right| \leq&~ (n_{\dist_0} + n_{\dist_1}) + \dtv(\dist_0(X), \dist_1(X)) \\
        &~+ \min\left\{\Exp_{\dist_0}[|h_0^* - h_1^*|], \Exp_{\dist_1}[|h_0^* - h_1^*|]\right\}.
    \end{align*}
\label{thm:accuracy}
\end{theorem}
\paragraph{Remark}
Theorem~\ref{thm:accuracy} upper bounds the discrepancy of accuracy across groups by three terms: the sum of group-wise noise, the distance of representations across groups and the discrepancy of the Bayes optimal classifiers. In an ideal setting where both distributions are noiseless, i.e., same individuals in the same group are always treated similarly, the upper bound simplifies to the latter two terms: 
\begin{equation*}
    \left|\err_{\dist_0}(h) - \err_{\dist_1}(h)\right| \leq \dtv(\dist_0(X), \dist_1(X)) + \min\left\{\Exp_{\dist_0}[|h_0^* - h_1^*|], \Exp_{\dist_1}[|h_0^* - h_1^*|]\right\}.
\end{equation*}
If we further require that the optimal decision functions $h_0^*$ and $h_1^*$ are close to each other, i.e., optimal decisions are insensitive to the group membership, then Theorem~\ref{thm:accuracy} implies that a sufficient condition to guarantee accuracy parity is to find group-invariant representation that minimizes $\dtv(\dist_0(X), \dist_1(X))$. Note, however, we should also pay attention to ensure that the first term as well as the third term in the upper bound do not increase drastically when learning the representations, as a change of the data representation will also change the noise term as well as the distance between the optimal decision functions based on the representations. We now present the proof for Theorem~\ref{thm:accuracy}:
\begin{proof}[Proof of Theorem~\ref{thm:accuracy}]
    First, we show that for $a\in\{0, 1\}$, $\err_{\dist_a}(h)$ cannot be too large if $h$ is close to $h_a^*$. Note that since we are focusing on binary classification problems, for any two classifiers $h, h'$, $\Pr(h(X)\neq h'(X)) = \Exp[|h(X) - h'(X)|]$.
    \begin{align*}
        \left|\err_{\dist_a}(h) - \Exp_{\dist_a}[|h(X) - h_a^*(X)|]\right| &= |\Exp_{\dist_a}[|h(X) - Y|] - \Exp_{\dist_a}[|h(X) - h_a^*(X)|]| \\
        &\leq  \Exp_{\dist_a}[\big||h(X) - Y| - |h(X) - h_a^*(X)|\big|]\\
        &\leq  \Exp_{\dist_a}[|Y- h_a^*(X)|] = n_{\dist_a},
    \end{align*}
    where both inequalities are due to the triangle inequality. Next, we bound $\left|\err_{\dist_0}(h) - \err_{\dist_1}(h)\right|$ by:
    \begin{equation*}
        \left|\err_{\dist_0}(h) - \err_{\dist_1}(h)\right| \leq (n_{\dist_0} + n_{\dist_1}) + \left| \Exp_{\dist_0}[|h(X) - h_0^*(X)|] - \Exp_{\dist_1}[|h(X) - h_1^*(X)|]\right|.
    \end{equation*}
    In order to show this, define $\varepsilon_{a}(h, h')\defeq \Exp_{\dist_a}[|h(X) - h'(X)|]$ so that
    \begin{equation*}
        \big| \Exp_{\dist_0}[|h(X) - h_0^*(X)|] - \Exp_{\dist_1}[|h(X) - h_1^*(X)|]\big| = \big|\varepsilon_0(h, h_0^*) - \varepsilon_1(h, h_1^*)\big|.
    \end{equation*}
    To bound $\big|\varepsilon_0(h, h_0^*) - \varepsilon_1(h, h_1^*)\big|$, realize that $|h(X) - h_a^*(X)|\in\{0, 1\}$. On one hand, we have:
    \begin{align*}
        \big|\varepsilon_0(h, h_0^*) - \varepsilon_1(h, h_1^*)\big| &= \big| \varepsilon_0(h, h_0^*) - \varepsilon_0(h, h_1^*) + \varepsilon_0(h, h_1^*) - \varepsilon_1(h, h_1^*)\big| \\
        &\leq \big| \varepsilon_0(h, h_0^*) - \varepsilon_0(h, h_1^*)\big| + \big|\varepsilon_0(h, h_1^*) - \varepsilon_1(h, h_1^*)\big| \\
        &\leq \varepsilon_0(h_0^*, h_1^*) + \dtv(\dist_0(X), \dist_1(X)),
    \end{align*}
    where the last inequality is due to $\big|\varepsilon_0(h, h_1^*) - \varepsilon_1(h, h_1^*)\big| = \big|\dist_0(|h - h_1^*| = 1) - \dist_1(|h - h_1^*| = 1)\big| \leq \sup_E|\dist_0(E) - \dist_1(E)| = \dtv(\dist_0, \dist_1)$. Similarly, by subtracting and adding back $\varepsilon_1(h, h_0^*)$ instead, we can also show that $\big|\varepsilon_0(h, h_0^*) - \varepsilon_1(h, h_1^*)\big| \leq \varepsilon_1(h_0^*, h_1^*) + \dtv(\dist_0(X), \dist_1(X))$. 

    Combine the above two inequalities yielding:
    \begin{equation*}
        \big|\varepsilon_0(h, h_0^*) - \varepsilon_1(h, h_1^*)\big| \leq \min\{\varepsilon_0(h_0^*, h_1^*), \varepsilon_1(h_0^*, h_1^*)\} + \dtv(\dist_0(X), \dist_1(X)).
    \end{equation*}
    Incorporating the sum of group-wise noise back to the above inequality by using one more triangle inequality then completes the proof.
\end{proof}

\section{Empirical Results}
\label{sec:experiment}
Our theoretical results on the lower bound imply that over-training the feature transformation function to achieve group-invariant representations will inevitably lead to large joint errors. On the other hand, our upper bound also implies that group-invariant representations help to achieve accuracy parity. To verify these theoretical implications, in this section we conduct experiments on a real-world benchmark dataset, the UCI Adult dataset, to present empirical results with various metrics. 

\paragraph{Dataset}
The Adult dataset contains 30,162/15,060 training/test instances for income prediction. Each instance in the dataset describes an adult from the 1994 US Census. Attributes include gender, education level, age, etc. In this experiment we use gender (binary) as the sensitive attribute, and we preprocess the dataset to convert categorical variables into one-hot representations. The processed data contains 114 attributes. The target variable (income) is also binary: 1 if $\geq$ 50K/year otherwise 0. For the sensitive attribute $A$, $A = 0$ means Male otherwise Female. In this dataset, the base rates across groups are different: $\Pr(Y = 1\mid A = 0) = 0.310$ while $\Pr(Y = 1\mid A = 1) = 0.113$. Also, the group ratios are different: $\Pr(A = 0) = 0.673$.

\paragraph{Experimental Protocol}
To validate the effect of learning group-invariant representations with adversarial debiasing techniques~\citep{zhang2018mitigating,madras2018learning,beutel2017data}, we perform a controlled experiment by fixing the baseline network architecture to be a three hidden-layer feed-forward network with ReLU activations. The number of units in each hidden layer are 500, 200, and 100, respectively. The output layer corresponds to a logistic regression model. This baseline without debiasing is denoted as NoDebias. For debiasing with adversarial learning techniques, the adversarial discriminator network takes the feature from the last hidden layer as input, and connects it to a hidden-layer with 50 units, followed by a binary classifier whose goal is to predict the sensitive attribute $A$. This model is denoted as AdvDebias. Compared with NoDebias, the only difference of AdvDebias in terms of objective function is that besides the cross-entropy loss for target prediction, the AdvDebias also contains a classification loss from the adversarial discriminator to predict the sensitive attribute $A$. In the experiment, all the other factors are fixed to be the same between these two methods, including learning rate, optimization algorithm, training epoch, and also batch size. To see how the adversarial loss affects the joint error, the demographic parity as well as the accuracy parity, we vary the coefficient $\rho$ for the adversarial loss between 0.1, 1.0, 5.0 and 50.0.

\paragraph{Results and Analysis}
The experimental results are listed in Table~\ref{tab:results}. Note that in the table $|\err_{\dist_0} - \err_{\dist_1}|$ could be understood as measuring an approximate version of accuracy parity, and similarly $\dpgap(\Ypred)$ measures the closeness of the classifier to satisfy demographic parity. From the table, it is then clear that with increasing $\rho$, both the overall error $\err_\dist$ (sensitive to the marginal distribution of $A$) and the joint error $\err_{\dist_0} + \err_{\dist_1}$ (insensitive to the imbalance of $A$) are increasing. As expected, $\dpgap(\Ypred)$ is drastically decreasing with the increasing of $\rho$. Furthermore, $|\err_{\dist_0} - \err_{\dist_1}|$ is also gradually decreasing, but much slowly than $\dpgap(\Ypred)$. This is due to the existing noise in the data as well as the shift between the optimal decision functions across groups, as indicated by our upper bound. To conclude, all the empirical results are consistent with our theoretical findings. 
\begin{table}[htb]
    \centering
    \caption{Adversarial debiasing on demographic parity, joint error across groups, and accuracy parity.}
    \label{tab:results}
    \begin{tabular}{lcccc}\toprule
        & $\err_\dist$  & $\err_{\dist_0} + \err_{\dist_1}$ & $\left|\err_{\dist_0} - \err_{\dist_1}\right|$ & $\dpgap(\Ypred)$\\\midrule
    NoDebias & 0.157 & 0.275 & 0.115 & 0.189 \\
    AdvDebias, $\rho = 0.1$ & 0.159 & 0.278 & 0.116 & 0.190 \\
    AdvDebias, $\rho = 1.0$ & 0.162 & 0.286 & 0.106 & 0.113 \\
    AdvDebias, $\rho = 5.0$ & 0.166 & 0.295 & 0.106 & 0.032 \\
    AdvDebias, $\rho = 50.0$ & 0.201 & 0.360 & 0.112 & 0.028 \\
    \bottomrule
    \end{tabular}
\end{table}

\section{Related Work}
\label{sec:related}
\paragraph{Tradeoff between Fairness and Accuracy}
Although it has long been empirically observed that there is an inherent tradeoff between accuracy and statistical parity in both classification and regression problems~\citep{calders2009building,zafar2015fairness,zliobaite2015relation,corbett2017algorithmic,zhao2020conditional,zhao2021costs}, precise characterizations on such tradeoffs are less explored. \citet[Proposition 8]{menon2018cost} explored such tradeoff in terms of the fairness frontier function under the context of cost-sensitive binary classification. In this work the fair machine learning problem is reduced to learning a classifier
which optimizes a difference between cost-sensitive risks, one with respect to the target variable and one with respect to the sensitive variable. \citet{zhao2019inherent} proved a lower bound of accuracy on both the sum of group-wise errors as well as the joint error that has to be incurred by any fair algorithm satisfying statistical parity. In this paper, assuming oracle access to Bayes optimal classifiers, we also give an algorithm to construct an optimal fair classifier that can verify the lower bound. Furthermore, we also extend the preliminary result in~\citet{zhao2019inherent} for binary classification and binary protected attribute to the general multi-class classification setting where the protected attribute can take more than two values, i.e., there are more than two groups defined by the protected attribute. Recently, \citet{chzhen2020fair} and \citet{le2020projection} concurrently derived an analytic bound to characterize the price of statistical parity in regression when the learner can take the sensitive attribute explicitly as an input for $\ell_2$ loss. In this case, the lower bound is given by the optimal transportation distance from two group distributions to a common one, characterized by the $W_2$ barycenter. Our result complements this line of works for the classification setting, where we show that the price paid by a fair classifier for general multi-class classification problems with a categorical protected attribute is given by the so-called TV-Barycenter problem. 

On the upside, under certain data generative assumptions of the sampling bias, there is a line of recent works showing that fairness constraints could instead improve the accuracy of the predictor~\citep{dutta2020there,blum2020recovering}. In particular, \citet{blum2020recovering} prove that if the observable data are subject to labeling bias, then the Equality of Opportunity constraint could help recover the Bayes optimal classifier. Note that this does not contradict with our results, since in this work we do not make any assumptions on the underlying training distributions, and we mainly focus on statistical parity and accuracy parity, rather than equalized odds.

\paragraph{Regularization Techniques}
The line of work on fairness-aware learning through regularization dates at least back to~\citet{kamishima2012fairness}, where the authors argue that simple deletion of sensitive features in data is insufficient for eliminating biases in automated decision making, due to the possible correlations among attributes and sensitive information~\citep{lum2016statistical}. In light of this, the authors proposed a \emph{prejudice remover} regularizer that essentially penalizes the mutual information between the predicted goal and the sensitive information. In a more recent approach,~\citet{zafar2015fairness} leveraged a measure of decision boundary fairness and incorporated it via constraints into the objective function of logistic regression as well as support vector machines. As discussed in Section~\ref{sec:preliminary}, both approaches essentially reduce to achieving demographic parity through regularization. 

\paragraph{Fair Representations}
In a pioneer work, \citet{zemel2013learning} proposed to preserve both group and individual fairness through the lens of representation learning, where the main idea is to find a good representation of the data with two competing goals: to encode the data for accuracy maximization while at the same time to obfuscate any information about membership in the protected group. Due to the power of learning rich representations offered by deep neural nets, recent advances in building fair automated decision making systems focus on using adversarial techniques to learn fair representation that also preserves enough information for the prediction vendor to achieve his accuracy~\citep{edwards2015censoring,louizos2015variational,beutel2017data,zhang2018mitigating,adel2019one,song2019fair,zhao2020conditional}. \citet{madras2018learning} further extended this approach by incorporating reconstruction loss given by an autoencoder into the objective function to preserve demographic parity, equalized odds, and equal opportunity.

\section{Discussion and Conclusion}
In this paper we theoretically and empirically study the important problem of quantifying the tradeoff between accuracy and statistical parity in algorithmic fairness. Specifically, we prove a novel lower bound to characterize the tradeoff between statistical parity and the joint accuracy across different population groups when the base rates differ between groups. In particular, our results imply that, in general, any method aiming to satisfy statistical parity admits an information-theoretic lower bound on the joint error. This holds even only approximate statistical parity is met. In light of this impossibility result, under the statistical parity constraint, we can only hope to design algorithms that achieve the accuracy lower bound. To this end, assuming oracle access to the potentially unfair Bayes classifiers, we construct an algorithm that returns a randomized classifier, and we prove that this randomized classifier is both optimal (in terms of accuracy) and fair. 

When the number of groups defined by the protected attribute is more than two, we can no longer obtain an analytic form of the lower bound. Nevertheless, we show that it can be efficiently computed by solving a linear program in polynomial time, which we term as the TV-Barycener problem. This finding also builds a connection between the tradeoff problem in algorithmic fairness and the barycenter problem (under the TV-distance) in optimal transport. Complementary to our negative results, we also show that learning fair representations leads to accuracy parity if the Bayes optimal classifiers across different groups are close. Our theoretical findings are also confirmed empirically on a real-world dataset. We believe our results take an important step towards better understanding the tradeoff between accuracy and different notions of fairness.

\acks{
HZ and GG would like to acknowledge support from the DARPA XAI project, contract \#FA87501720152 and a Nvidia GPU grant. HZ would also like to thank support from a Facebook research award. The authors are very grateful to the anonymous reviewers for the suggestions on improving the presentation of this work.}

\newpage
\bibliography{reference}







\end{document}